\documentclass[letterpaper]{article} % DO NOT CHANGE THIS
\usepackage{aaai23}  % DO NOT CHANGE THIS
\usepackage{times}  % DO NOT CHANGE THIS
\usepackage{helvet}  % DO NOT CHANGE THIS
\usepackage{courier}  % DO NOT CHANGE THIS
\usepackage[hyphens]{url}  % DO NOT CHANGE THIS
\usepackage{graphicx} % DO NOT CHANGE THIS
\usepackage{natbib}  % DO NOT CHANGE THIS AND DO NOT ADD ANY OPTIONS TO IT
\usepackage{caption} % DO NOT CHANGE THIS AND DO NOT ADD ANY OPTIONS TO IT
\usepackage[T1]{fontenc}
\usepackage{amsmath,amsfonts,bm}

\def\eqref#1{equation~\ref{#1}}
\def\1{\bm{1}}

\DeclareMathAlphabet{\mathsfit}{\encodingdefault}{\sfdefault}{m}{sl}
\SetMathAlphabet{\mathsfit}{bold}{\encodingdefault}{\sfdefault}{bx}{n}

\usepackage[hidelinks]{hyperref}
\usepackage{url}
\usepackage[ruled]{algorithm2e}
\usepackage{graphicx}
\usepackage{multirow}
\usepackage{tabularx}
\usepackage{booktabs}
\usepackage{amsmath}
\usepackage{enumerate}
\usepackage{enumitem}
\usepackage{calc}
\usepackage{subcaption}
\usepackage{mathtools}
\usepackage{multirow}
\usepackage[dvipsnames]{xcolor}
\usepackage{amsthm}

\newtheorem*{lemma}{Lemma}

\newcommand{\tell}{T^{\ell}}
\newcommand{\telli}{T^{\ell}_i} 
\newcommand{\thetaelli}{\theta^{\ell}_i}

\title{Invariant Representations with Stochastically Quantized Neural Networks}

\author {
    Mattia Cerrato \textsuperscript{\rm 1}\equalcontrib ,
    Marius Köppel \textsuperscript{\rm 2}\equalcontrib,
    Roberto Esposito \textsuperscript{\rm 3},
    Stefan Kramer \textsuperscript{\rm 1}
}
\affiliations {
    \textsuperscript{\rm 1} Institute of Computer Science, Johannes Gutenberg-Universität Mainz, Germany \\
    \textsuperscript{\rm 2} Institute for Nuclear Physics, Johannes Gutenberg-Universität Mainz, Germany \\
    \textsuperscript{\rm 3} Computer Science Department, Università degli Studi di Torino, Italy \\
    \{mcerrato, mkoeppel\}@uni-mainz.de, roberto.esposito@unito.it, kramer@informatik.uni-mainz.de
}

\begin{document}

\maketitle

\begin{abstract}

Representation learning algorithms offer the opportunity to learn invariant representations of the input data with regard to nuisance factors.
Many authors have leveraged such strategies to learn fair representations, i.e., vectors where information about sensitive attributes is removed. These methods are attractive as they may be interpreted as minimizing the mutual information between a neural layer's activations and a sensitive attribute.
However, the theoretical grounding of such methods relies either on the computation of infinitely accurate adversaries or on minimizing a variational upper bound of a mutual information estimate.
In this paper, we propose a methodology for direct computation of the mutual information between neurons in a layer and a sensitive attribute. We employ stochastically-activated binary neural networks, which lets us treat neurons as random variables.
Our method is therefore able to minimize an upper bound on the mutual information between the neural representations and a sensitive attribute.
We show that this method compares favorably with the state of the art in fair representation learning and that the learned representations display a higher level of invariance compared to full-precision neural networks.

\end{abstract}

\section{Introduction}\label{sec:intro}

Representation learning algorithms based on neural networks are being employed extensively in information retrieval and data mining applications.
The social impact of what the general public refers to as ``AI'' is now a topic of much discussion, with regulators in the EU even putting forward legal proposals which would require practitioners to ``[...] minimize the risk of unfair biases embedded in the model [...]'' \cite{aiproposal}. 
Such proposals refer to biases with regard to individual characteristics which are protected by the law, such as gender and ethnicity.
The concern is that models trained on biased data might then learn those biases \cite{barocas-hardt-narayanan}, therefore perpetuating historical discrimination against certain groups of individuals. 
Machine learning methodologies designed to avoid these situations are often said to be ``group-fair''.

%In recent years, research on \emph{group fairness} in machine learning has focused on formalizing definitions and measuring model bias \cite{verma2018fairness}, introducing concepts as disparate impact, disparate mistreatment, statistical parity, etc..
%As an example, a classification model might display \emph{disparate impact} if it assigns positive outcomes (e.g. getting a loan) with different rates to different groups of people (e.g. men and women).
%\emph{Disparate mistreatment}, on the other hand, is the situation where a classifier misassigns negative outcomes with different rates across groups \cite{zafar2017fairness}.

\begin{figure*}
    \centering
    \includegraphics[width=0.72\linewidth]{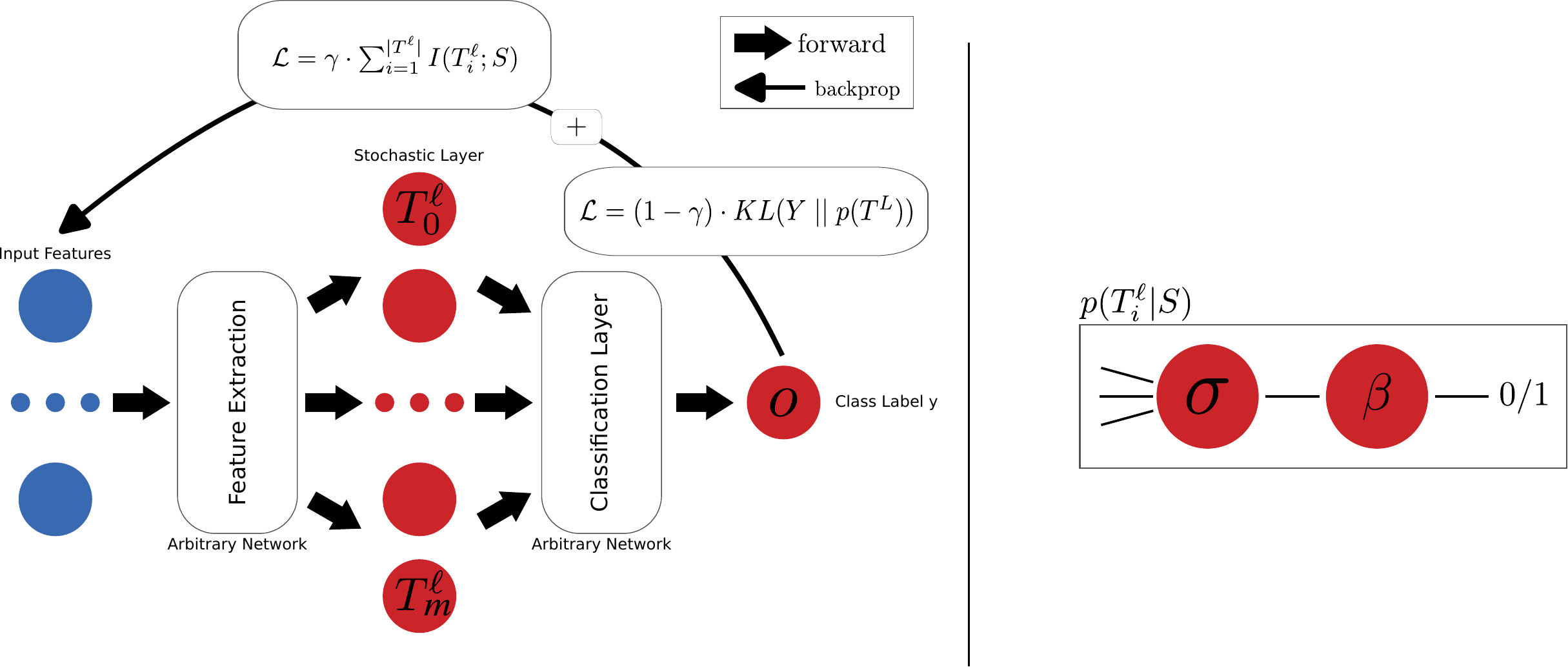}
    \caption{\textbf{Left}: sketch of a Stochastically Quantized Neural Network. The stochastic layer $\tell$, holding the quantized neurons, is shown with $T_0^{\ell} \dots T_m^{\ell}$. During the forward step the input features are extracted via feature extraction layers before they enter the stochastic layer. After the stochastic layer further classification layers are used to classify the class label $y$. During the backpropagation the loss function is evaluated e.g. via binary cross entropy for the class label $y$ and by calculating the mutual information of the stochastic layer $T^{\ell}$ and the sensitive attribute $S$. The influence of the two loss functions is controlled via the parameter $\gamma$. The feature extraction and classification layers may be chosen to be full-precision. \textbf{Right}: sketch of a stochastically quantized neuron. The neuron is sigmoid-activated, but the sigmoid output is employed as the parameter for a Bernoulli distribution, which we then sample from. This enables the interpretation of $\telli$ as a random variable and $\tell$ as a random vector, in turn allowing us to compute $I(\tell;S)$.}
    \label{fig:model}
\end{figure*}

One possible approach to the group fairness problem is fair representation learning.  
Fair representation learning is a set of techniques which learns new representations of the original data where information about protected characteristics such as gender or ethnicity has been removed.
Numerous authors have employed neural networks as the base learning algorithm in this task \cite{zemel2013learning,xie2017controllable,madras}.
The core concept in fair representation learning is to remove sensitive information from a neural network's representations.
If one is able to guarantee that the network has no information about, for instance, an individual's gender identity, it follows that the decisions undertaken by the model are independent of it.
One possible formalization for the above desideratum relies on computing the mutual information between the $\ell$-th neural layer $\tell$ and a sensitive attribute $S$ representing the protected information. 
If the mutual information $I(\tell; S)$ equals $0$, for instance, one would obtain group-invariant representations and therefore unbiased decisions which do not rely on $S$.
Estimating mutual information over high-dimensional distributions is however highly challenging in general.
In practice, an upper bound on $I(\tell;S)$ is sufficient for minimization purposes.
Previous research has employed post-training quantization \cite{tishby2015deep}, variational approximations of the encoding distribution \cite{moyer2018invariant,belghazi2018mutual} and adversarial learning \cite{madras,xie2017controllable}.
Also, the mutual information $I(\tell;S)$ is ill-defined in neural networks as they are deterministic functions of the input data. 
Implying that $I(X;\tell)$ is vacuous, it does not depend on the network parameters \cite{goldfeld2019estimating}.

In this work, we provide an alternative approach to fair representation learning by employing stochastically-quantized neural networks.
These networks have a low-precision activation (1 bit) which may be interpreted as a random variable following the Bernoulli distribution.
We show a sketch of these networks in Figure~\ref{fig:model}.
Thus, we are able to obtain a finite, exact value for the mutual information $I(\tell_i;S)$ between a neuron $i$ in any layer $\ell$ and the sensitive attribute, without relying on variational approximation \cite{moyer2018invariant}, adversarial learning \cite{madras} or adding noise to the representations \cite{goldfeld2019estimating}.
We then employ this exact value to upper bound the overall mutual information in layer $\ell$, $I(\tell; S)$, obtaining an optimization objective which leads to invariant representations.
Furthermore, we discuss how to compute $I(\tell; S)$ in stochastically activated binary neural networks via density estimation.
%As the network is stochastic, an arbitrary number of samples may be obtained independently of how much training data is available.
Our experimentation proves that quantized, low-precision models are able to obtain group-invariant representations in a fair classification setting and avoid training set biases in image data. 
Our contributions can be summarized as follows: I. We show how to compute the mutual information $I(\tell_i; S)$ between a stochastic binary neuron and the sensitive attribute $S$. II. We show how to use this value to upper bound the mutual information for the whole layer $I(\tell; S)$, which is a natural objective for invariant representation learning. III. We employ density estimation to compute the mutual information $I(\tell; S)$. IV. We perform experiments on three fair classification datasets and one invariant representation learning dataset, showing that our low-precision model is competitive with the state of the art.
% \begin{itemize}
%     \item We show how to compute the mutual information between a stochastic binary neuron $i$ in layer $\ell$ and a discrete random variable $S$ representing a sensitive attribute $I(\tell_i; S)$.
%     \item We employ the computation above and an independence assumption to compute the mutual information for the whole layer $I(\tell; S)$. This computation may be used as an optimization objective.
%     \item We relax the assumption above and employ density estimation to compute the mutual information $I(\tell; S)$.
%     \item We perform experiments on three fair classification datasets and one invariant representation learning dataset, showing that our low-precision model is competitive with the state of the art.
% \end{itemize}

\section{Related Work}\label{sec:related}

Algorithmic fairness has attracted significant attention from the academic community and general public in recent years, thanks in no small part to the ProPublica/COMPAS debate \cite{machine_bias,Rudin2020Age}.
%. COMPAS is a software which was intended to assist US courtroom judges in evaluating the risk of recidivism in people seeking to be released on bail. ProPublica \cite{machine_bias} reported the COMPAS software as being discriminatory towards black people. More specifically, ProPublica found that the false positive rates were different for black and white people. We refer the reader to a contribution by Rudin et al. for a summary of the debate \cite{Rudin2020Age}. 
To the best of our knowledge, however, the first contribution in this area dates back to 1996, when Friedman and Nisselbaum \cite{friedman1996bias} contributed that automatic decision systems need to be developed with particular attention to systemic discrimination and moral reasoning. The need to tackle automatic discrimination is also part of EU-level law in the GDPR, Recital 71 in particular \cite{malgieri2020gdpr}.

One possible way to tackle these issues is to remove the information about the ``nuisance factor'' $S$ from the data $X$ by employing fair representation learning.
%One possible way to tackle the issues described above is to employ fair representation learning. In fair representation learning, sensitive and law-protected information (ethnicity, gender identity, etc.) is treated as a ``nuisance factor'' $S$ which needs to be removed from the data $X$. While excluding $S$ in training could be regarded as a ``color-blind'' approach to fairness (we owe this definition to Zehlike et al. \cite{zehlike2018reducing}), it is often not enough to obtain unbiased decisions, as complex statistical correlations between $X$ and $S$ may still exist. 
Fair representation learning techniques learn a projection of $X$ into a latent feature space $Z$ where all information about $S$ has been removed. One seminal contribution in this area is due to Zemel et al. \cite{zemel2013learning}. Since then, neural networks have been extensively used in this space. Some proposals \cite{xie2017controllable, madras} employ adversarial learning, a technique due to Ganin et al. \cite{ganin2016jmlr} in which two networks are pitted against one another in predicting $Y$ and removing information about $S$. Another line of work \cite{Louizos2016TheVF, moyer2018invariant} employs variational inference to approximate the intractable distribution $p(Z \mid X)$. A combination of architecture design \cite{Louizos2016TheVF} and information-theoretic loss functions \cite{moyer2018invariant, gretton2012kernel} may then be employed to encourage invariance of the neural representations with regard to $S$. Our proposal differs in that we employ a stochastically quantized neural network for which it is possible to compute the mutual information between the $i$-th neuron at layer $\ell$ and the sensitive attribute $S$. This lets us avoid using approximations of the target distribution $p(\tell, S)$ and provides a more stable training objective for representation invariance compared to adversarial training.
More recently, neural architectures have been proposed for other fairness-related settings such as fair ranking \cite{zehlike2018reducing,cerrato2020pairwise,fair_pair_metric} and fair recourse \cite{shubham2021fair}.
Our experimental comparison will focus, however, on the fair/invariant classification setting so to enable a comparison with other methods in this area which bound or approximate information measures in neural representations.

\section{Method}

Our contribution deals with learning fair (group-invariant) representations in a principled way by employing stochastically quantized neural layers.
In this section, we provide a theoretical motivation for our work by contextualizing it in an information-theoretic framework similar to the one introduced in the ``Information Bottleneck'' literature \cite{tishby2015deep, goldfeld2019estimating}. % (Section~\nameref{sec:estimation}).
As previously mentioned, the work done so far in this space has approximated measuring the information theoretic quantities in neural networks either via post-training quantization \cite{tishby2015deep}, adversarial bounding \cite{madras}, variational approximations \cite{moyer2018invariant} or the addition of stochastic noise in the representations \cite{goldfeld2019estimating,cerrato2020constraining}.
These approximations are necessary as the mutual information is an ill-defined concept in deterministic neural nets \cite{goldfeld2020information}.
Our approach is instead to employ stochastically quantized neural layers to exactly compute the mutual information between a sensitive attribute $S$ and any neuron in the layer.
%We show how this approach may be used to compute group-invariant representations in Section \nameref{sec:model} and \nameref{sec:joint}.

\subsection{Invariant Representations and Mutual Information}\label{sec:estimation}
A feedforward neural network with $L$ layers may be formalized as a sequence of ``layer functions'' $\phi^{\ell}$ which compute the neural activations given an input $x \in \mathbb{R}^{d_0}$:
\begin{align}
    \phi^{\ell}(x) &= \sigma(A^{\ell} \phi^{\ell-1}(x) + b^{\ell}), & \ell &= 1, \dots, L \label{eq:phi-comp} \\
    \phi^{0}(x) &= x,
\end{align}
where $A_{\ell} \in \mathbb{R}^{d_{\ell} \times d_{\ell-1}}$ is a weight matrix, $b \in \mathbb{R}^{d_{\ell}}$ is a bias vector, $\sigma$ is an activation function, and $d_{\ell}$ is the size of the $\ell$-th layer.
We now define neural representations as applications of $\phi$ to the random variable $X$ which follows the empirical distribution of the $x$ samples:
\begin{align*}
    T^{\ell} &:= \phi^{\ell}(X), & \ell = 1, \dots, L
\end{align*}
We note that, in a supervised learning setting, the last layer $T^L$ is a reproduction of $Y$. 
Representation invariance may be formalized as an information-theoretic objective in which the representations of the $\ell$-th neural layer $\tell$ display minimal mutual information with regard to a sensitive attribute $S$:

\begin{equation}\label{eq:min_mi}
    I(T^{\ell}; S) \leq \alpha,
\end{equation}
with $\alpha \approx 0$ being a threshold where the information can be considered minimal, and $I(\tell; S)$ is the mutual information between $\tell$ and $S$.
Mutual information is an attractive measure in invariant representation learning, as two random variables are statistically independent if and only if $I(A;B) = 0$.
Thus, one may obtain $S$-invariant representations by minimizing the mutual information between $S$ and $\tell$, therefore certifying that no information about the sensitive attribute is contained in the representations.
Minimizing this objective might, however, remove all information about the original data $X$ and, possibly, the labels $Y$.
To avoid this issue, one might want to guarantee that some information about $X$ is preserved. 
This changes the setting to a problem which closely resembles the Information Bottleneck problem \cite{tishby2015deep} in which the task is to minimize $I(T^{\ell}; S)$ s.t. $I(T^{\ell}; X) \geq \beta$,
% \begin{align*}
%     & \text{minimize} && I(T^{\ell}; S) \\ 
%     & \textbf{s.t.} && I(T^{\ell}; X) \geq \beta,
% \end{align*}
where $\beta$ is a positive real number.
However, constrained optimization is highly challenging in neural networks: in practice, previous work in this area has focused on, e.g., minimizing a reconstruction loss as a surrogate for the $I(\tell; X) \geq \beta$ constraint \cite{madras}.

At the same time, computing $I(\tell; S)$ is very challenging.
Since the distribution of $\tell$ is non-parametric and in general unknown, one would need to resort to density estimation to approximate it from data samples.
However, the number of needed samples scales exponentially with the dimensionality of $\tell$ \cite{paninski2003estimation}.
Furthermore, mutual information is in general ill-defined in neural networks: if $X$ is a random variable representing the empirical data distribution and $f$ is a deterministic injective function, the mutual information $I(X;f(X))$ is either constant (when $X$ is discrete) or infinite (when $X$ is continuous) \cite{goldfeld2019estimating}.
As activation functions in neural networks such as \texttt{sigmoid} and \texttt{tanh} are injective and real-valued, $I(\tell; S)$ is infinite. When the \texttt{ReLU} activation function is employed, the mutual information is finite but still vacuous -- i.e. independent of the network's parameters\footnote{Proper contextualization of this result requires some additional preliminary results, which we will avoid reporting here due to space constraints. We refer the interested reader to Goldfeld and Polyanskiy \cite{goldfeld2020information}.}.

Previous work in fair representation learning has circumvented this issue by employing different techniques.
One possible approach is to perform density estimation by grouping the real-valued activations into a finite number of bins \cite{tishby2015deep}.
This approach, however, returns different values for the mutual information depending on the number of bins \cite{tishby2015deep} and obtains the actual, ``true'' mutual information value only as the number of bins approaches infinity \cite{goldfeld2020information}.  
Some authors have instead relied on variational approximations of the intractable distribution $p(\tell | X)$, which leads to a tractable upper bound \cite{moyer2018invariant}.
Lastly, it is possible to bound the mutual information term $I(\tell; S)$ with the loss of an adversary network which tries to predict $S$ from $\tell$ \cite{madras,ganin2016jmlr,xie2017controllable}. 
The adversary is, however, usually chosen to be a deterministic network, and thus this result suffers from the issues described in this section. 

Our approach is instead to employ stochastically-quantized neural layers, a technique used in binary networks \cite{binarynet}, enabling the interpretation of neural activations as (discrete) random variables.
This approach has multiple benefits: It does not rely on variational approximations or adversarial training; it lets us treat $\tell$ as a random variable, avoiding the infinite mutual information issue described above; lastly, it avoids adding noise to the representations as a way to obtain stochasticity \cite{goldfeld2019estimating}.

\subsection{Mutual Information Computation via Bernoulli activations}\label{sec:model}

%\textbf{AS: write down that we do early stopping to find the best model - counter the problem with having no information about y}

Our methodology relies on stochastically quantized neural layers in which the activations are stochastic.
More specifically, we employ \emph{binary} neural layers in which either the weights or the activations have 1-bit precision.
In a binary network, $\telli$ may be computed either deterministically or stochastically from the activations of the previous layer $T^{\ell-1}$ and the learnable weights and biases $w$ connecting the two. For the sake of presentation, we report in the following the formula for deterministic computation of $\tell$:

\begin{equation*}
    T^{\ell}_{i} = \mathbb{I} \left[ \frac{1}{\mid T^{\ell-1} \mid}\sum_{i=1}^{\mid T^{\ell-1} \mid} w_i T^{\ell - 1}_{i} \geq 0.5 \right],
\end{equation*}
where $\mathbb{I}$ is the indicator function.
In this paper, however, we employ stochastic quantization of binary neurons, which we compute by sampling from $\mathcal{B}(\theta^l_i)$, where:
\begin{align}\label{eq:theta}
    \theta^{\ell}_{i} &= \sigma\left(\frac{1}{\mid T^{\ell-1} \mid}\sum_{i=1}^{\mid T^{\ell-1} \mid} w_i T^{\ell - 1}_{i}\right),
\end{align}
where $\sigma$ is the sigmoid function. % $\sigma(x) = \frac{1}{1+e^{-x}} \in [0, 1]$. 
We then sample from this distribution to compute the actual activation $\telli$.
Thus, $\telli$ may be interpreted as a random variable following the Bernoulli distribution: $T^{\ell}_i \sim \mathcal{B}(\theta^{\ell}_{i})$. The entropy of a random variable following the Bernoulli distribution has a closed, analytical form:
\begin{equation}
    \label{eq:entropy_bernoulli}
    H(\telli) = -(1-\thetaelli) \cdot \log_2(1-\thetaelli) - \thetaelli \cdot \log_2(\thetaelli).
\end{equation}
Recalling that the definition of mutual information may be rewritten in terms of entropy as $I(\telli;S) = H(\telli) - H(\telli \mid S)$, we then note that it is possible to compute it exactly for a sensitive attribute $S$ and a neuron $\telli$. The conditional entropy $H(\tell_i \mid S=s)$ may be computed by selecting those representations $\tell_i = \phi(x)$ for which it is true that $S=s$.

\noindent We then consider the whole layer as a stochastic random vector $\tell = [\tell_1, \tell_2, ... ,\tell_m]$.
We show in the following lemma that $\sum_{i=1}^{\mid \tell \mid} I(\telli;S) \geq I(\tell;S)$. 
Thus, minimizing $\sum_{i=1}^{\mid \tell \mid} I(\telli ; S)$ will minimize $I(\tell ; S)$.

\begin{lemma}
Let $\tell$ be a random vector of a given layer in the network, $S$ a random variable representing the sensitive attribute, $|\tell|$ the number of neurons in $\tell$ and $TC(\tell;S)$ the informativeness~\cite{informativeness}. Then, the mutual information $I(\tell; S)$ is minimized if $\sum^{|\tell|}_{i=1}I(\telli; S)$ is minimized.
\end{lemma}

\begin{proof}
We recall the definition of total correlation \cite{watanabe} for a random vector $A = (a_1, \dots, a_n)$ having the probability density function $p(A)$:
\begin{align}\label{eq:tc-definition}   
    TC(A) &= \sum_{i=1}^n H(a_i) - H(A) \\
    &= KL[p(A) \mid \mid p(a_1) \, p(a_2) \dots \, p(a_n)] \nonumber
\end{align}
We then define the informativeness of the sensitive attribute and a layer $\tell$ in the network as:
\begin{equation}\label{informativeness}
    TC(\tell;S) = TC(\tell) - TC(\tell | S),
\end{equation}
where $TC(\tell)$ is the total correlation~\cite{totalcorrelation1} or multi-information~\cite{totalcorrelation2} of $\tell$ and $TC(\tell|S)$ the conditional total correlation of $\tell$ given $S$.
One can rewrite the total correlation for $\tell$ and $S$ as entropy terms following Equation~\ref{eq:tc-definition}. Similarly, the conditional total correlation may be reformulated as:
\begin{equation}\label{eq:conditionaltotalcorrelation}
    TC(\tell|S) = \sum^{|\tell|}_{i=1}H(\telli|S) - H(\tell|S).
\end{equation}
Following the derivation of Gao et al.~\cite{informativeness} Equation~\ref{informativeness} can be rewritten, using Equation~\ref{eq:tc-definition} and Equation~\ref{eq:conditionaltotalcorrelation}, to:
\begin{align*}
    TC(\tell;S) &= \sum^{|\tell|}_{i=1}H(\telli) - H(\tell) \\
                & - \sum^{|\tell|}_{i=1}H(\telli|S) + H(\tell|S) \\
%                 &= \sum^{|\tell|}_{i=1}I(\telli; S) - H(\tell) + H(\tell|S) \\
                &= \sum^{|\tell|}_{i=1}I(\telli; S) - \underbrace{(H(\tell) - H(\tell|S))}_{I(\tell; S)}.
\end{align*}
Since $TC(\tell;S)\geq 0$, $I(\telli; S)\geq 0$ and $I(\tell; S)\geq 0$ it follows that:
\begin{equation}\label{eq:final-relation}
    \sum^{|\tell|}_{i=1}I(\telli; S) \geq I(\tell; S).
\end{equation}
\end{proof}

Thus, given a stochastically quantized binary neural layer, we can impose representation invariance with respect to $S$ by stochastic gradient descent over a loss function $\mathcal{L}$ that directly incorporates $I(\telli;S)$:
\begin{align}
    \mathcal{L} = \gamma \cdot \sum_{i=1}^{\mid \tell \mid} I(\telli; S) + (1-\gamma) \cdot KL(Y \mid \mid p(T^{L})),
\end{align}\label{eg:loss-function}
\noindent where $KL$ is the Kullback-Leibler divergence and $\gamma$ is a trade-off parameter weighting the importance of representation invariance and accuracy on $Y$.
It is worthwhile to mention that only a single layer $\ell$ needs to be stochastically quantized for the mutual information term to be computable.
We explored the performance of both a ``hybrid'' network with a mix of binary-precision and full-precision layers and a fully binary-precision network (see Table~\ref{tab:best-hp} in the supplementary material).
Nevertheless, we found in our experiments that the ``hybrid'' method has the strongest performers in all cases.

%The independence assumption, which is critical to the derivation above, warrants some discussion, however.
% \textcolor{red}{MK: We don't have this independence assumption anymore.}
% While the sampling of each Bernoulli variable is indeed independent, the parameters of each underlying distribution $\thetaelli$ are not independent.
% This is due to them depending from the activations at the previous layer $T^{\ell -1}$, as it is made apparent by Equation~\ref{eq:theta}.
% Thus, independence holds only conditionally with regard to the Bernoulli parameter vector $\theta^{\ell}$, e.g., $\tell_1 \bot \tell_2 \mid \theta^{\ell}_1, \theta^{\ell}_2$.
% Independence assumptions are relatively common in practice in variational approximation, where, e.g., the encoding distribution $q(\tell | X)$ is sampled from an isotropic Gaussian distribution \cite{Louizos2016TheVF}. Our methodology is able to relax this constraint by taking advantage of the quantization intrinsic to our base model. This is the topic of section \nameref{sec:joint} in the supplementary.

We also note that stochastically quantized binary neural networks are a natural fit for density estimation techniques.
As the network is stochastic and quantized by default, it is possible to estimate the conditional distribution $p(\tell \mid S)$ from samples while avoiding the issues with post-training quantization and ill-definedness described earlier in this section and in the literature \cite{goldfeld2019estimating,goldfeld2020information}.
We do not give specifics about this method in this section, as it relies on standard counting techniques.
However, we give a full description of this strategy in the supplementary material.
We test this method in the next section, where we dubbed it \texttt{BinaryMI} to contrast it with the methodology described previously in this section, which we will refer to as \texttt{BinaryBernoulli}.
\section{Experiments}\label{sec:experiments}
Our experimental setting focuses on analyzing the performance of the methodology presented in this paper in fair classification and invariant representation learning settings. 
Our experimentation aims to answer the following questions:

\textbf{Q1.} Is the present methodology able to learn fair models which avoid discrimination?
\textbf{A1.} Yes. We analyze our models' accuracy and fairness by measuring the area under curve (AUC) and their disparate impact / disparate mistreatment. We compare them with full-precision neural networks trained for fair classification and see that our approach is able to find strong accuracy/fairness tradeoffs under the assumption that both are equally important. We also observe that supervised classifiers trained with the learned representations and the sensitive data are unable to generalize to the test set, as their accuracy is close to random guessing performance. Furthermore, we show that our proposed models are able to remove nuisance factors from image datasets.

\textbf{Q2.} Is the present methodology \emph{stable}, i.e. is it able to learn different tradeoffs between accuracy and fairness?
\textbf{A2.} Yes. Our method is able to explore the fairness-accuracy tradeoff with a higher degree of stability than previously proposed neural models when changing the tradeoff parameter $\gamma$. We observe high positive correlation between higher $\gamma$ and higher fairness for both \texttt{BinaryBernoulli} and \texttt{BinaryMI}. 

%The datasets we employed are presented in Section \nameref{sec:datasets}. We then discuss the fairness metrics in Section \nameref{sec:metrics}. Finally, we describe the general experimental setup in Section \nameref{sec:generalexp}. 

\begin{figure*}
    \includegraphics[width=\linewidth]{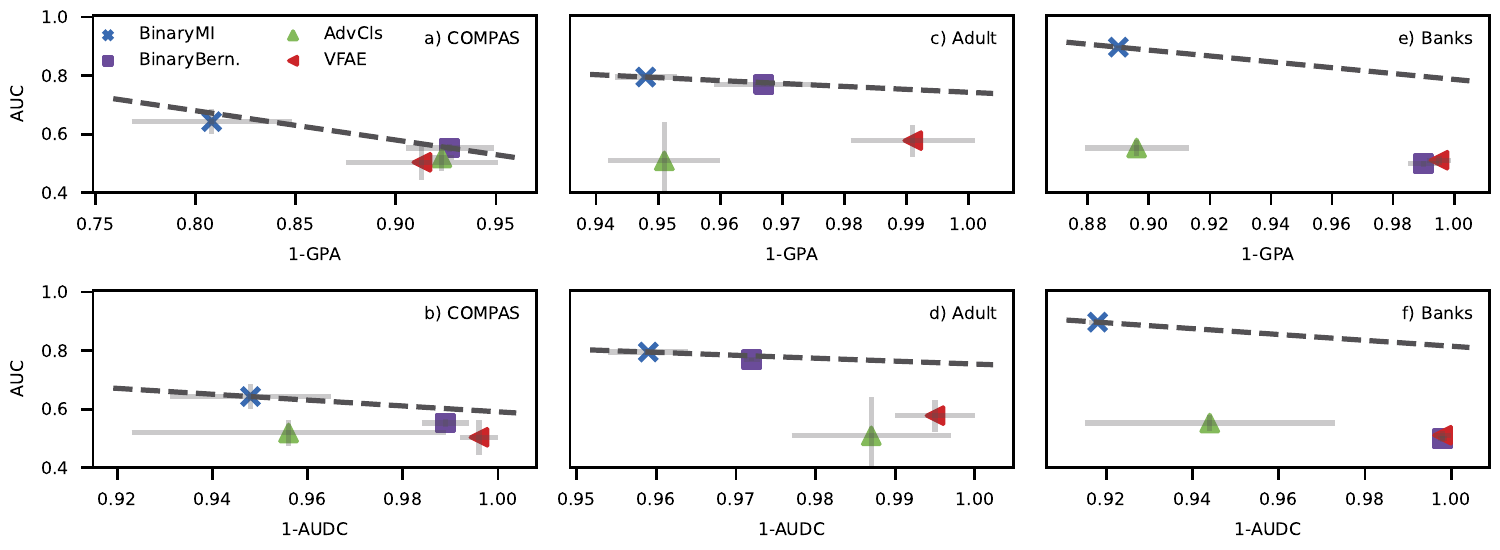}
    \caption{Experiment results for our Bernoulli entropy model (\texttt{BinaryBernoulli}), our density estimation model (\texttt{BinaryMI}), an adversarial classifier (\texttt{AdvCls}), and a fair variational autoencoder (\texttt{VFAE}). The dotted line represents the line of equivalent fairness/accuracy tradeoffs and goes through the model closest (under the $L_1$ norm) to perfect accuracy and perfect fairness. We show on top our best performing models for the AUC/1-GPA tradeoff, and on the bottom the best performing models for AUC/1-AUDC. \label{fig:results}}
\end{figure*}

\subsection{Datasets}\label{sec:datasets}

\textbf{COMPAS.} This dataset was released by ProPublica \cite{machine_bias}.%, an US-based consortium \cite{machine_bias}, as part of an empirical study on ML software which US judges employ to evaluate the risk of further crimes by individuals who have been previously arrested. 
The ground truth is whether an individual committed a crime in the following two years. 
The sensitive attribute is the individual's ethnicity.
Machine learning models trained on this dataset may display disparate mistreatment \cite{zafar2017fairness}, thus our objective is minimizing an equal opportunity metric, the Group-dependent Pairwise Accuracy, while maximizing accuracy (AUC).\\
\textbf{Adult.} This dataset is part of the UCI repository \cite{dua2019uci}. 
The ground truth is whether an individual's annual salary is over 50K\$ per year or not \cite{adult}.
This dataset has been shown to be biased against gender \cite{Louizos2016TheVF,zemel2013learning}.\\
\textbf{Bank marketing}. In this dataset, the classification goal is whether an invidivual will subscribe a term deposit.
Models trained on this dataset may display both disparate impact and disparate mistreatment with regard to age, more specifically on whether individuals are under 25 and over 65 years of age.\\
\textbf{Biased-MNIST}. This is an image dataset based on the well-known MNIST Handwritten Digits database in which the background has been modified so to display a color bias \cite{bahng2020learning}.
More specifically, a nuisance factor $C$ is introduced which is highly correlated with the ground truth $Y$ and whose values represent the background color in the training set. 
Ten different colors are pre-selected for each value of $Y = \{0 ... 9\}$ and inserted as a background in the training images with high probability ($p = 0.99, 0.995$). 
The test images, on the other hand, have background color chosen at random.
We show samples from the training and test data in Figure~\ref{fig:mnist}, which may be found in the supplementary material.
Thus, the background color/nuisance factor provides a very strong training bias. The simplest strategy for a model to achieve high accuracy on the training set is to overfit the background color. Therefore, models that are unable to learn invariant representations and decisions will inevitably overfit the training set \cite{bahng2020learning}.

\subsection{Metrics}\label{sec:metrics}

\paragraph*{Group-dependent Pairwise Accuracy.}\label{sec:gpa}\ We employ this to test the disparate mistreatment of our models. Its introduction is due to Narashiman et al. \cite{fair_pair_metric}. We report its definition in the following.  

Let $G_1, ..., G_K$ be a set of $K$ groups such that every instance inside the dataset $\mathfrak D$ belongs to one of these groups. The \emph{group-dependent pairwise accuracy} $A_{G_i > G_j}$ is then defined as the accuracy of a classifier on instances which are labeled as positives and belong to group $G_i$ and instances labeled as negatives which belong to group $G_j$. Group-dependent pairwise accuracy is then defined as $|A_{G_i > G_j} - A_{G_j > G_i}|$, and should be close to zero.
The rationale for this metric is that the false positive and false negative rates should be equalized across groups if possible, similarly to the notions of equality of opportunity \cite{hardt2016equality} or disparate mistreatment \cite{zafar2017fairness}.
In the following, we call the Group-dependent Pairwise Accuracy {\em GPA}. 

\paragraph*{Area under Discrimination Curve (AUDC)}
We take the discrimination as a measure of disparate impact \cite{zemel2013learning}, which is given by:
\begin{equation*}
\text{yDiscrim} = \left | \frac{\sum^n_{n:s_n=1} \hat{y}_n}{\sum^n_{n:s_n=1} 1} - \frac{\sum^n_{n:s_n=0} \hat{y}_n}{\sum^n_{n:s_n=0} 1} \right |, 
\end{equation*}
where $n:s_n=1$ denotes that the $n$-th example has a value of $s$ equal to 1.
We then generalize this metric in a similar fashion to how accuracy may be generalized to obtain a classifier's area under the curve (AUC): We evaluate the measure above for different classification thresholds and then compute the area under this curve.
We employed $100$ equi\-spaced thresholds in our experiments.
In the following, we will refer to this measure as AUDC (area under the discrimination curve) as done elsewhere in the literature~\cite{interFair}.
Contrary to AUC, lower values are better. 

\subsection{Experimental setup}\label{sec:generalexp}
We split all datasets into 3 internal and 3 external folds.
On the 3 internal folds, we employ a Bayesian optimization technique to find the best hyperparameters for our model.
A summary of our models' best hyperparameters can be found in the supplementary material at Table \ref{tab:best-hp}.
As our interest is to obtain models which are both fair and accurate, we employ Bayesian optimization to maximize the sum of the models' AUC, 1-GPA and 1-AUDC. 
We set the maximum number of iterations to $200$.
The best hyperparameter setting found this way is then evaluated on the 3 external folds and reported. We relied on the Weights \& Biases platform for an implementation of Bayesian optimization and overall experiment tracking \cite{wandb}. 
On the fairness datasets, we compare with an adversarial classifier (AdvCls in the Figures) trained as described by Xie et al. \cite{xie2017controllable} and a variational fair autoencoder (VFAE) \cite{Louizos2016TheVF}. We obtained publicly available implementations for these models and optimized the hyperparameters with the same strategy we employed for our models. We report our model relying on the Bernoulli entropy computation as \verb|BinaryBernoulli| in the figures, while the model relying on density estimation to compute the joint $p(\tell)$ is referred to as \verb|BinaryMI|.

\subsection{Fair and Invariant Classification}
\label{sec:results}

We report plots analyzing the accuracy/fairness tradeoff of our models trained for classification in Figure~\ref{fig:results}. We take AUC as our measure of accuracy and both 1-GPA and 1-AUDC as the fairness metric. The ideal model in the fair classification setting displays maximal AUC and little to no GPA/AUDC and would appear at the very top right in Figure~\ref{fig:results}. This result is not attainable on the datasets we consider, as there is usually some correlation between $S$ and $Y$, which prevents us to obtain perfectly fair and accurate decisions. Thus, one needs to consider possible accuracy/fairness tradeoffs. We assume a balanced accuracy/fairness tradeoff and consider as the ``best'' model the one closest to $(1, 1)$ under the $L_1$ norm. We then show all equivalent tradeoff points as a dotted line. We see that on both COMPAS and Adult, our \verb|BinaryMI| model is able to find a stronger tradeoff than the competitors, with \verb|BinaryBernoulli| very close to an equivalent tradeoff. The same may be said for the third dataset, Banks, in which however our two models find very different tradeoffs, with \verb|BinaryBernoulli| preferring an almost perfectly fair result and \verb|BinaryMI| finding a very accurate model. Compared to adversarial learning and variational inference, our models either find the best tradeoff or lie closest to the best tradeoff line.
We also analyze the accuracy/fairness tradeoff of the representations learned by our models. We extract neural activations from all the networks considered at the penultimate layer $T^{\ell-1}$, which we stochastically quantized in all experiments. We then report in Figure 4 in the supplementary material the performance of a Random Forest algorithm with 1000 base estimators trained to predict the sensitive attribute $S$ associated with each representation. We observe that \texttt{BinaryMI} is able to find the best tradeoff between informativeness on $Y$ (bottom row)
and invariance to $S$. Our best \texttt{BinaryBernoulli} model representations, which are comparatively not very invariant on COMPAS, performed strongly in this regard on both Adult and Banks.
% \begin{figure}
%     \centering
%     \includegraphics[width=0.4\textwidth]{pics/zplot_gamma.pdf}
%     \caption{1-GPA vs. AUC plot for \texttt{BinaryBernoulli}, \texttt{AdvCls} and  \texttt{VFAE} while varying the $\gamma$ parameter. The $z$ values are shown by smoothed lines using splines to fit the data for each method. The data used for these experiments was the COMPAS dataset.}
%     \label{fig:gamma_advcls}
% \end{figure}
% 
% \begin{figure}
%     \centering
%     \includegraphics[width=0.4\textwidth]{pics/complexity_CompasCls.pdf}
%     \caption{1-GPA vs. AUC plot for \texttt{BinaryBernoulli}, \texttt{AdvCls} and  \texttt{VFAE} while varying the model complexity by changing the number of hidden layers. The $z$ values are shown by smoothed lines using splines to fit the data for each method. The data used for these experiments was the COMPAS dataset.}
%     \label{fig:complexity_compas}
% \end{figure}

\begin{figure}
    \centering
    % \begin{subfigure}[b]{0.47\textwidth}
        % \centering
    \includegraphics[width=.5\textwidth]{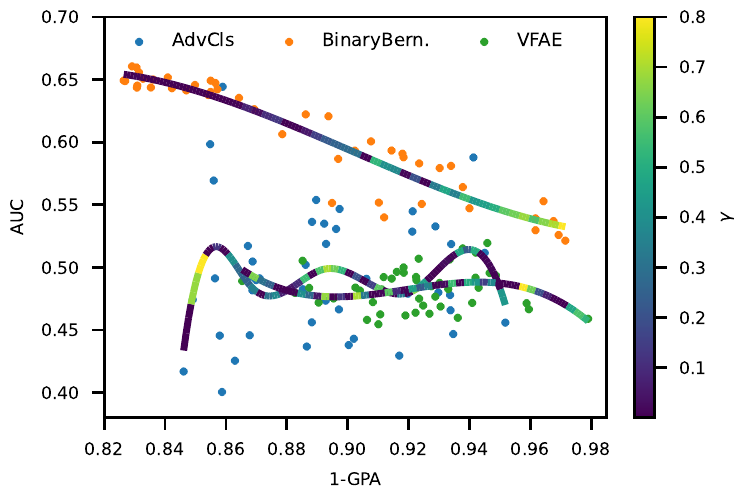}
        % \caption{Stability Analysis}
    % \end{subfigure}
    % \begin{subfigure}[b]{0.49\textwidth}
    %     \centering
    %     \includegraphics[width=\textwidth]{pics/complexity_CompasCls.pdf}
    %     \caption{Complexity Analysis}
    %     \label{fig:complexity_compas}
    % \end{subfigure}
    \caption{1-GPA vs. AUC plot for \texttt{BinaryBernoulli}, \texttt{AdvCls} and  \texttt{VFAE} while varying the $\gamma$ parameter. The values are regressed with smoothed lines using splines fitted on the result for each method separately. The color of the spline segments represents the value of $\gamma$ for the specific run. The data used for these experiments was the COMPAS dataset.}
    \label{fig:gamma_advcls}
    % \label{fig:stability_ana}
\end{figure}

\subsection{Stability and Complexity Analysis}

It is important that fair classification models are able to find different tradeoffs between fairness and accuracy depending on the application requirements.
This tradeoff may be regulated, in practical terms, via a parameter $\gamma$ which weights the importance of accuracy and invariance in the loss function.
This idea is common, to the best of our knowledge, to all the fair representation learning algorithms developed so far.
We explore how the performance of our model changes with $\gamma$ in Figure~\ref{fig:gamma_advcls}, where we report the performance of a \texttt{BinaryBernoulli} model trained on the COMPAS dataset, an adversarial classifier and a fair variational autoencoder for comparison.
What we observe in Figure~\ref{fig:gamma_advcls} is that our model displays a relatively stable performance. As $\gamma$ grows, so does the fairness of the model in terms of 1-GPA (linear correlation $\rho=0.924$). The model is highly sensitive to values of $\gamma$ between $0.01$ and $0.5$. The same trend, but reversed, may be observed for AUC ($\rho = -0.836$).
We observed similar trends for \texttt{BinaryMI}, with correlations of $0.781$ and $-0.760$ for 1-GPA and AUC, respectively.
Thus, we reason that our proposal may be employed under different fairness requirements with minimal changes (a tweak of the $\gamma$ parameter). The adversarial classifier, on the other hand, displays little correlation between $\gamma$ and its performance.
While the best models are on par or almost on par with \texttt{BinaryBernoulli}, this happens for arbitrary values of $\gamma$. The adversarial model does not seem to be able to explore the fairness/accuracy tradeoff quite as well, and the effect of the $\gamma$ parameter is unpredictable. We posit that this behavior may be due to the difficulty of striking a balance between the predictive power of the two sub-networks which predict $Y$ and $S$ alternatively \cite{xie2017controllable}. This is a well-known issue for generative adversarial models which pit different networks against each other \cite{Chu2020Smoothness}.
We also note that the variational approximation-based model (VFAE) struggles to come to an accurate result, whereas it mostly takes fair decisions. 
The correlation between $\gamma$, which controls the strength of the Maximum Mean Discrepancy regularization in this model, and AUC/1-GPA is also quite low ($0.173$ and $-0.168$ respectively).

\subsection{Biased-MNIST} \label{sec:mnist}
\begin{table}
\centering
\caption{Accuracies for the biased MNIST experiments. Results for other methodologies as reported by Bahng et al. \cite{bahng2020learning}. We report results for two bias levels $\rho=0.99$ and $0.995$. A higher value of $\rho$ implies a higher chance of a biased sample in the training set.}
\resizebox{0.47\textwidth}{!}{
\begin{tabular}{cccccccc}
\toprule
 $\rho$ & Vanilla & ReBias & LearnedMixin & RUBi & BinaryMI & BinaryBernoulli \\
 \midrule
  0.995 & 72.1  & 76.0 & 78.2 & 90.4 & 89.08 & 90.64\\
  0.990 & 89.1  & 88.1 & 88.3 & 93.6 & 88.54 & 96.02 \\
\bottomrule
\end{tabular}
}
\label{tab:mnist_exp}
\end{table}

We report our method's performance on the Biased-MNIST dataset in Table~\ref{tab:mnist_exp}. We also report results from Bahng et al. \cite{bahng2020learning} as a comparison. To enable this comparison, we experimented with the same setup as the authors' by training our model for 80 epochs. As usual we selected our best hyperparameters with a Bayesian optimization strategy employing an Alexnet-like structure \cite{krizhevsky2012imagenet} by alternating (binary) convolutional layers and max-pooling layers.
We report the full information for the best hyperparameters we found in Table~\ref{tab:best-hp} in the supplementary material. 
We then tested with two different bias levels, i.e. the probability of a training sample displaying a specific color bias.
We observe that \texttt{BinaryBernoulli} is the strongest performer on both bias levels, even when compared with other full-precision strategies. \texttt{BinaryMI} has a comparable performance to a ``vanilla'' convolutional network when the probability of training bias is $\rho = 0.99$. However, it displays better scaling to the higher bias level than the baseline method. 
In this dataset, we see that our methodology is also a strong performer when removing biases from image data is necessary for classification accuracy.

\section{Conclusion and Future Work}

In this paper we proposed a methodology to compute the mutual information between a stochastically activated neuron and a sensitive attribute. 
We then generalized this methodology into two different strategies to compute the mutual information between a layer of neural activations and a sensitive attribute.
Both our strategies perform strongly on both fair classification datasets and invariant image classification.
Furthermore, our methodology displays high stability to changes of the accuracy/fairness tradeoff parameter $\gamma$, especially when compared to adversarial learning \cite{xie2017controllable, madras}. 
A possible direction for further development is to employ the methodologies discussed in this paper to revisit the debate on the information bottleneck problem introduced by Tishby et al. \cite{tishby2015deep}. 
As our models are partly stochastically quantized, they naturally lend themselves to mutual information computation, avoiding many of the common issues in estimating information measures in neural networks \cite{goldfeld2019estimating}.
While estimating the conditional probability $p(\tell \mid S)$ does seem to scale to very wide networks (i.e. layers with many neurons), it is possible to repeat the sampling (the stochastic quantization) as many times as needed, which could alleviate this issue. 
Furthermore, we would like to test the capabilities of the presented models in domain adaptation scenarios, where adversarial models are still used extensively \cite{ganin2016jmlr, madras}.

\newpage

\section{Acknowledgements}
We thank Alexander Segner for useful discussions throughout the development of this paper. The paper was partially supported by the ``TOPML: Trading Off Non-Functional Properties of Machine Learning'' project funded by the Carl-Zeiss-Stiftung in the Förderprogramm ``Durchbrüche'', identifying code P2021-02-014. This work has been partially supported by the European PILOT project and by the HPC4AI project \cite{hpc4ai}.

\clearpage
\section{Supplementary Material}\label{sec:supplementary}

% \begin{algorithm}
%  \SetKwRepeat{Do}{do}{while}
%  \caption{\texttt{BinaryBernoulli}}
%  \KwData{$(x, y)$ training data; $s$ sensitive data; $\gamma$ invariance tradeoff hyperparameter; $\ell \in [1, L)$ position of the stochastic binary layer}
%  \Repeat{convergence}{
%      j := 0 \\
%      \While{$j < L$} {
%          Compute $\phi_{j}(x)$ as in Eq. \ref{eq:phi-comp} \\
%          j := j + 1
%      }
%      Compute $\theta^{\ell}$ as in Eq. \ref{eq:theta} \\
%      $\tell$ := \texttt{sample(}$\theta^{\ell}$\texttt{)} \\
%      Compute $I(\tell; S)$ as in Eq. \ref{eq:final-mi-indep} \\
%      Compute $\mathcal{L}$ as in Eq. \ref{eg:loss-function} \\
%      Apply Stochastic Gradient Descent 
%   }
%   \label{algo}
% \end{algorithm}

\subsection{Mutual Information Computation via Density Estimation}\label{sec:joint}
In the following, we describe a strategy to compute the mutual information in a stochastically-quantized neural network via density estimation.
In general, given two joint continuous random variables $A$ and $B$ taking values over the sample space $\mathcal{A} \times \mathcal{B}$, the mutual information between them is defined as 

\begin{equation}\label{eq:mi}
    I(A;B) = \int_{\mathcal{B}} \int_{\mathcal{A}} p_{(A, B)}(a, b) \; log (\frac{p_{(A, B)}(a, b)}{p_A(a)p_B(b)}) d\mu_{(A, B)},
\end{equation}
where is the measure related to $p_{(A, B)}$. We note that the integrals in Equation~\ref{eq:mi} may be replaced by sums if $A$ and $B$ are discrete random variables. 
Mutual information may also be expressed in terms of entropy:
\begin{equation}\label{eq:condentro}
    I(A;B) = H(A) - H(A \mid B)
\end{equation}
We report here the general formula for entropy $H(X)$ when $X$ is a discrete variable:
\begin{equation}
    H(X) = \sum_{x \in X} -p(x) \; log \; p(x) \label{eq:entropy}
\end{equation}
Recalling the definition of Mutual Information in Equation \ref{eq:mi} and its relationship to entropy in Equation \ref{eq:condentro}, we see that an estimate of the joint $p(\tell) = p(\tell_1, \dots, \tell_m)$ and the joint conditional $p(\tell_1, \dots \tell_m \mid S)$ is needed to compute the mutual information $I(\tell;S)$ in a network. 
This is highly nontrivial in full-precision neural networks. In that situation, $\tell$ is not a random variable and each neuron $\telli$ may take any of $2^p$ values, where $p$ is the precision the activations are computed with. 
In our setup, however, each neuron $\telli$ may only take values in $\{0, 1\}$. Thus, we are able to estimate the joint distribution of the random vector $\tell = [\tell_1 ... \tell_m]$ from data via a simple density estimation scheme: plainly put, we count the occurrences of each possible activation vector $\tell$.

In general, density estimation may be performed, in its simplest form, via building a histogram for the data at hand $x=[x_1, \dots, x_n]$. Let us define $h$ as the histogram function, $B$ as a binning function and $\delta$ as the bin width. Formally, for $x \in \mathbb{R}$,  $h$  and $B$ can be defined as:

\begin{align*}
    h(x) &= \sum_{i=1}^{n} B(x - \hat{x}_i;\delta) \\
    B(x,\delta) &= 
    \begin{dcases}
    1 &x \in (-\delta/2, \delta/2) \\
    0 &\text{otherwise}
    \end{dcases}
\end{align*}
where $i$ iterates over the data samples and $\hat{x}_i$ is the center of the bin where $x_i$ lies.
Given the function above, one is able to estimate the probability density function for a data sample by computing $p(x) = \frac{1}{n\delta}h(x)$. Density estimation may also be computed with adaptive, data-dependant bin sizes or by using kernels to replace the binning function $B$. Parametric density estimation may also be employed when one has prior knowledge about the data following some parametric distribution, e.g. a Gaussian. These techniques have been extensively employed when estimating mutual information for high-dimensional distributions~\cite{kraskov2004estimating} such as neural activations. 
In our setup, however, we employ a stochastic quantization scheme which returns binary values. This lets us abstract away the choice of $\delta$, which is known to have a major influence in mutual information estimation in neural networks \cite{goldfeld2019estimating,goldfeld2020information}. 
Assuming for simplicity of presentation that $\tell = [\tell_1, \tell_2]$, we are able to estimate $p(\tell) = p(\tell_1, \tell_2)$ by counting the frequencies of all possible activation vectors, i.e. the realizations of the random vector $\tell \in \{[0,0],[0,1],[1,0],[1,1]\}$, or, equivalently, the output of the layer function $\phi^{\ell}(X)$. The general density computation for a layer of size $m$ follows:

\begin{align}
    B_i(\tell) &= 
    \begin{dcases}
        1 &\phi^{\ell}(X) = i \\ 
        0 &\text{otherwise}
    \end{dcases} \label{eq:abuse} \\ 
    h_i(\tell) &= \frac{1}{n}\sum_{j=1}^n B_i(\tell) \nonumber \\
    p(\tell) &= [h_1(\tell), \dots, h_m(\tell)] \nonumber
\end{align}

Where, with some abuse of notation, $\phi^{\ell}(X) = i$ indicates that a realization of the activation vector $\tell$ is equal to the $i$-th element of the sequence\footnote{Here, any ordering of the set is viable. Thus, we avoid specifying the sequence formally.} of binary vectors associated with the set $\{0, 1\}^m$. This procedure may also be employed to compute $p(\tell \mid S=s)$ by only selecting those data samples for which $S=s$ and adjusting the normalization factor $1/n$ accordingly. 
Therefore, we are able to compute the mutual information by estimating the underlying probabilities.

One limitation of the methodology presented in this subsection is that the number of samples required to estimate the joint distribution $p(\tell)$ scales exponentially with the number of neurons in the layer $m$. In practice, $m$ may be chosen as an hyperparameter and kept as low as needed. However, a very small sized layer placed at any point in the network will limit the network's expressivity. Another possible solution to this matter is to perform multiple forward passes so to obtain better estimates for $p(\tell)$ and $p(\tell \mid S)$, which we however did not investigate presently.

\begin{figure*}
    \includegraphics[width=\linewidth]{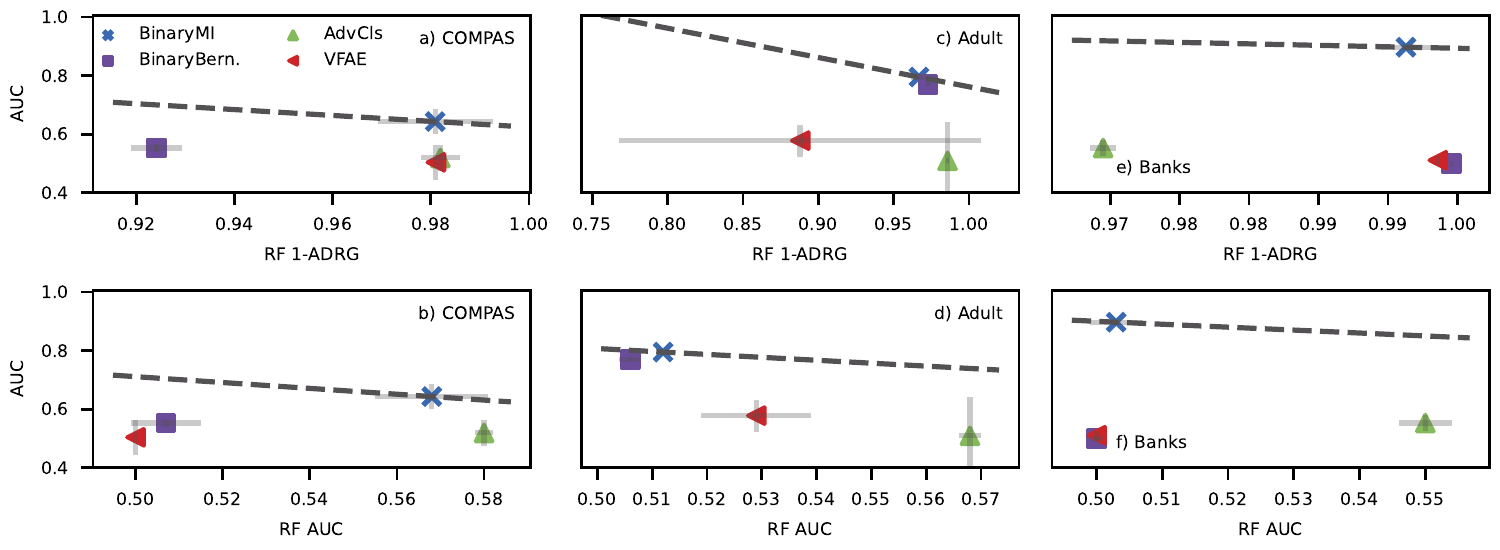}
    \caption{Representation results for our Bernoulli entropy model (\texttt{BinaryBernoulli}), our joint density estimation model (\texttt{BinaryMI}), an adversarial classifier (\texttt{AdvCls}) and a fair variational autoencoder (\texttt{VFAE}). The dotted line represent the line of equivalent fairness/accuracy tradeoffs and goes through the model closest (under the $L_1$ norm) to (1, 1). A Random Forest (RF) classifier is trained on the extracted representations $\tell$. In the top row, model AUC is compared with RF 1-ADRG, i.e. the absolute distance to random guess. This is computed by subtracting the RF accuracy to the majority class ratio in the dataset and taking the absolute value. In the bottom row, we compare the model AUC and the RF AUC. \label{fig:repr}}
\end{figure*}

\subsection{Hyperparameters}

We include our models' best hyperparameters in Table~\ref{tab:best-hp}.
Our data was split into 3 internal and 3 external folds for a total of 9 fits. 
On the 3 internal folds, we employ a Bayesian optimization technique to find the best hyperparameters for our model.
As our interest is to obtain models which are both fair and accurate, we employ Bayesian optimization to maximize the sum of the models' AUC, 1-GPA and 1-AUDC. 
We set the maximum number of iterations of the Bayesian optimization algorithm to $200$.
The best hyperparameter setting found this way is then evaluated on the 3 external folds and reported in Table~\ref{tab:best-hp}.
\begin{table}
\resizebox{0.49\textwidth}{!}{
    \begin{tabular}{llcc|}
    \cline{3-4}
                                                             & \multicolumn{1}{l|}{}       & \multicolumn{1}{c|}{\texttt{BinaryMI}} & \texttt{BinaryBernoulli} \\ \hline
    \multicolumn{1}{|l|}{\multirow{3}{*}{$\gamma$}}          & \multicolumn{1}{l|}{COMPAS} & \multicolumn{1}{c|}{$0.16$}   & $0.23$                                    \\ \cline{2-4} 
    \multicolumn{1}{|l|}{}                                   & \multicolumn{1}{l|}{Banks}  & \multicolumn{1}{c|}{$0.03$}   & $0.60$                                    \\ \cline{2-4} 
    \multicolumn{1}{|l|}{}                                   & \multicolumn{1}{l|}{Adult}  & \multicolumn{1}{c|}{$0.01$}   & $0.04$                                    \\ \hline
    \multicolumn{1}{|l|}{\multirow{3}{*}{Batch size}}        & \multicolumn{1}{l|}{COMPAS} & \multicolumn{1}{c|}{$242$}    & $175$                                     \\ \cline{2-4} 
    \multicolumn{1}{|l|}{}                                   & \multicolumn{1}{l|}{Banks}  & \multicolumn{1}{c|}{$153$}    & $240$                                     \\ \cline{2-4} 
    \multicolumn{1}{|l|}{}                                   & \multicolumn{1}{l|}{Adult}  & \multicolumn{1}{c|}{$228$}    & $225$                                     \\ \hline
    \multicolumn{1}{|l|}{\multirow{3}{*}{N. hidden layers}}  & \multicolumn{1}{l|}{COMPAS} & \multicolumn{1}{c|}{$2$}      & $3$                                       \\ \cline{2-4} 
    \multicolumn{1}{|l|}{}                                   & \multicolumn{1}{l|}{Banks}  & \multicolumn{1}{c|}{$5$}      & $1$                                       \\ \cline{2-4} 
    \multicolumn{1}{|l|}{}                                   & \multicolumn{1}{l|}{Adult}  & \multicolumn{1}{c|}{$4$}      & $3$                                       \\ \hline
    \multicolumn{1}{|l|}{\multirow{3}{*}{Hidden layer size}} & \multicolumn{1}{l|}{COMPAS} & \multicolumn{1}{c|}{$10$}     & $20$                                      \\ \cline{2-4} 
    \multicolumn{1}{|l|}{}                                   & \multicolumn{1}{l|}{Banks}  & \multicolumn{1}{c|}{$40$}     & $30$                                      \\ \cline{2-4} 
    \multicolumn{1}{|l|}{}                                   & \multicolumn{1}{l|}{Adult}  & \multicolumn{1}{c|}{$50$}     & $50$                                      \\ \hline
    \multicolumn{1}{|l|}{\multirow{3}{*}{Hybrid}}            & \multicolumn{1}{l|}{COMPAS} & \multicolumn{1}{c|}{yes}      & yes                                       \\ \cline{2-4} 
    \multicolumn{1}{|l|}{}                                   & \multicolumn{1}{l|}{Banks}  & \multicolumn{1}{c|}{yes}      & yes                                       \\ \cline{2-4} 
    \multicolumn{1}{|l|}{}                                   & \multicolumn{1}{l|}{Adult}  & \multicolumn{1}{c|}{yes}      & yes                                       \\ \hline
    \multicolumn{1}{|l|}{Learning Rate}                      & \multicolumn{3}{c|}{$0.0001$}                                                                           \\ \hline
    \multicolumn{1}{|l|}{Optimizer}                          & \multicolumn{3}{c|}{ADAM}                                                                               \\ \hline
    \multicolumn{1}{|l|}{Epochs}                             & \multicolumn{3}{c|}{100}                                                                                \\ \hline
    \end{tabular}
}
\caption{Best hyperparameter combinations for \texttt{BinaryMI} and \texttt{BinaryBernoulli}. Learning rate, optimizer and number of epochs were kept fixed for both models across datasets. \label{tab:best-hp}} 
\end{table}

\subsection{Fair and Invariant Representations}\label{sec:repr-results}

We include the results for our analysis of the accuracy/fairness tradeoff of \emph{the representations} learned by our models in Figure~\ref{fig:repr}. As mentioned in the paper, we extract neural activations from all the networks considered at the penultimate layer $T^{L-1}$. In Figure~\ref{fig:repr} the performance of a Random Forest algorithm with 1000 base estimators trained to predict the sensitive attribute $S$ associated with each representation is reported. We refer to the paper for a discussion of these results.

\subsection{Biased-MNIST Training Images}

We include an example of Biased-MNIST training and testing images, to clarify how a biased dataset may be constructed from MNIST images. 

\begin{figure}[h!]
\centering
\includegraphics[width=\linewidth]{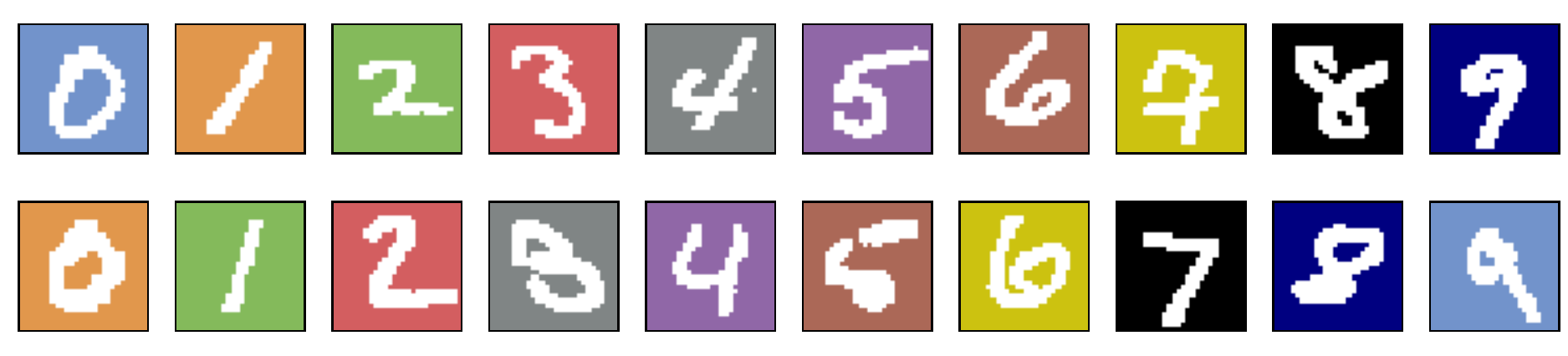}
\caption{Top row: Biased-MNIST training images. Each class is associated with a specific background color. Bottom row: testing images. The testing background color is selected at random. }
\label{fig:mnist}
\end{figure}

\end{document}